\documentclass[letterpaper]{article} 
\usepackage{aaai25}  
\usepackage{times}  
\usepackage{helvet}  
\usepackage{courier}  
\usepackage[hyphens]{url}  
\usepackage{graphicx} 
\usepackage{natbib}  
\usepackage{caption} 
\usepackage{algorithm}
\usepackage{algorithmic}

\usepackage{multirow}
\usepackage{amssymb}
\usepackage{amsmath}
\usepackage{amsthm}

\newtheorem{theorem}{Theorem}
\newtheorem{proposition}[theorem]{Proposition}
\newtheorem{lemma}[theorem]{Lemma}

\newtheorem{definition}[theorem]{Definition}

\usepackage{newfloat}
\usepackage{listings}
\DeclareCaptionStyle{ruled}{labelfont=normalfont,labelsep=colon,strut=off} 
\floatstyle{ruled}
\newfloat{listing}{tb}{lst}{}
\floatname{listing}{Listing}
\usepackage{hyperref} 

\title{Offline Multi-Agent Reinforcement Learning via\\ In-Sample Sequential Policy Optimization}
\author{
    Zongkai~Liu\textsuperscript{\rm 1, \rm3},
    Qian~Lin\textsuperscript{\rm 1},
    Chao~Yu\textsuperscript{\rm 1, \rm 2}\thanks{Corresponding author.}, 
    Xiawei~Wu\textsuperscript{\rm 1}, 
    Yile~Liang\textsuperscript{\rm 4}, 
    Donghui~Li\textsuperscript{\rm 4}, 
    Xuetao~Ding\textsuperscript{\rm 4}
}
\affiliations{
    \textsuperscript{\rm 1}School of Computer Science and Engineering, Sun Yat-sen University, Guangzhou, China\\
    \textsuperscript{\rm 2}Pengcheng Laboratory, Shenzhen, China\\
    \textsuperscript{\rm 3}Shanghai Innovation Institute, Shanghai, China\\
    \textsuperscript{\rm 4}Meituan, Beijing, China\\

    \{liuzk, linq67\}@mail2.sysu.edu.cn, {yuchao3}@mail.sysu.edu.cn
}

\usepackage{bibentry}

\begin{document}

\maketitle

\begin{abstract}

Offline Multi-Agent Reinforcement Learning (MARL) is an emerging field that aims to learn optimal multi-agent policies from pre-collected datasets. Compared to single-agent case, multi-agent setting involves a large joint state-action space and coupled behaviors of multiple agents, which bring extra complexity to offline policy optimization. 
In this work, we revisit the existing offline MARL methods and show that in certain scenarios they can be problematic, leading to uncoordinated behaviors and out-of-distribution (OOD) joint actions. 
To address these issues, we propose a new offline MARL algorithm, named In-Sample Sequential Policy Optimization (InSPO). InSPO sequentially updates each agent's policy in an in-sample manner, which not only avoids selecting OOD joint actions but also carefully considers teammates' updated policies to enhance coordination. Additionally, by thoroughly exploring low-probability actions in the behavior policy, InSPO can well address the issue of premature convergence to sub-optimal solutions. Theoretically, we prove InSPO guarantees monotonic policy improvement and converges to quantal response equilibrium (QRE). Experimental results demonstrate the effectiveness of our method compared to current state-of-the-art offline MARL methods.
%
\begin{links}
    \link{Code}{https://github.com/kkkaiaiai/InSPO/}
\end{links}

\end{abstract}

\section{Introduction}

Offline Reinforcement Learning (RL) is a rapidly evolving field that aims to learn optimal policies from pre-collected datasets without interacting directly with the environment~\cite{offlineSurvey/abs-2203-01387}. The primary challenge in offline RL is the issue of distributional shift~\cite{ICQ/YangMLZZHYZ21}, which occurs when policy evaluation on out-of-distribution (OOD) samples leads to the accumulation of extrapolation errors. Existing research usually tackles this problem by employing conservatism principles, compelling the learning policy to remain close to the data manifold through various data-related regularization techniques~\cite{ICQ/YangMLZZHYZ21, OMAR/PanH0X22, alberdice/MatsunagaLYLAK23, cfcql/ShaoQCZJ23, omiga/WangXZZ23}.

In comparison to the single-agent counterpart, offline Multi-Agent Reinforcement Learning (MARL) has received relatively less attention. Under the multi-agent setting, it not only faces the challenges inherent to offline RL but also encounters common MARL issues, such as difficulties in coordination and large joint-action spaces~\cite{marlSurvey}. 
These issues cannot be simply resolved by combining state-of-the-art offline RL solutions with modern multi-agent techniques~\cite{ICQ/YangMLZZHYZ21}. In fact, due to the increased number of agents and the offline nature of the problem, these issues become even more challenging. For example, under the offline setting, even if each agent selects an in-sample action, the resulting joint action may still be OOD. 
Additionally, in cooperative MARL, agents need to consider both their own actions and the actions of other agents in order to determine their contributions to the global return for high overall performance. Thus, under offline settings, discovering and learning cooperative joint policies from the dataset poses a unique challenge for offline MARL.

To address the aforementioned issues, recent works have developed specific offline MARL algorithms. These approaches generally integrate the conservatism principle into the Centralized Training with Decentralized Execution (CTDE) framework, such as value decomposition structures~\cite{ICQ/YangMLZZHYZ21, OMAR/PanH0X22, alberdice/MatsunagaLYLAK23, cfcql/ShaoQCZJ23, omiga/WangXZZ23}, which is developed under the Individual-Global-Max (IGM) assumption.
Although these approaches have demonstrated successes in certain offline multi-agent tasks, they still exhibit several limitations. 
For example, due to the inherent limitations of the IGM principle, algorithms that utilize value decomposition structures may struggle to find optimal solutions because of constraints in their representation capabilities, and can even lead to the selection of OOD joint actions, as we show in the Proposed Method section.

In this work, we propose a principled approach to tackle OOD joint actions issue. 
By introducing a behavior regularization into the policy learning objective and derive the closed-form solution of the optimal policy, we develop a sequential policy optimization method in an entirely in-sample learning manner without generating potentially OOD actions.
Besides, the sequential update scheme used in this method enhances both the representation capability of the joint policy and the coordination among agents. 
Then, to prevent premature convergence to local optima, we encourage sufficient exploration of low-probability actions in the behavior policy through the use of policy entropy.
The proposed novel algorithm, named the \underline{In}-Sample \underline{S}equential \underline{P}olicy \underline{O}ptimization (InSPO), enjoys the properties of monotonic improvement and convergence to quantal response equilibrium (QRE)~\cite{QRE/mckelvey1995quantal}, a solution concept in game theory. 
We evaluate InSPO in the XOR game, Multi-NE game, and Bridge to demonstrate its effectiveness in addressing OOD joint action and local optimum convergence issues. Additionally, we test it on various types of offline datasets in the StarCraft II micromanagement benchmark to showcase its competitiveness with current state-of-the-art offline MARL algorithms.

\section{Related Work}

\paragraph{MARL. }

The CTDE framework dominates current MARL research, facilitating agent cooperation. In CTDE, agents are centrally trained using global information but rely only on local observations to make decisions. Value decomposition is a notable method, representing the joint Q-function as a combination of individual agents' Q-functions~\cite{qplex/WangRLYZ21, QTRAN/SonKKHY19, QMIX/RashidSWFFW18}. These methods typically depend on the IGM principle, assuming the optimal joint action corresponds to each agent's greedy actions. However, environments with multi-modal reward landscapes frequently violate the IGM assumption, limiting the effectiveness of value decomposition in learning optimal policies~\cite{whyNoIGM/FuYXYW22}.

Another influential class of methods is Multi-Agent Policy Gradient (MAPG), with notable algorithms such as MAPPO\cite{MAPPO/YuVVGWBW22}, CoPPO\cite{CoPPO/WuYYZPZ21}, and HAPPO~\cite{HAPPO/KubaCWWSW022}. However, on-policy learning approaches like these struggle in offline settings due to OOD action issues, leading to extrapolation errors.

\paragraph{Offline MARL. }

OMAR~\cite{OMAR/PanH0X22} combines Independent Learning and zeroth-order optimization to adapt CQL~\cite{CQL/KumarZTL20} for multi-agent scenarios. However, OMAR fundamentally follows a single-agent learning paradigm, which treats other agents as part of the environment, and does not handle cooperative behavior learning and OOD joint actions insufficiently.

To enhance cooperation and efficiency in complex environments such as StarCraft II, some existing works employ value decomposition as a foundation for algorithm design.
For instance, ICQ~\cite{ICQ/YangMLZZHYZ21} introduces conservatism to prevent optimization on unseen state-action pairs, mitigating extrapolation errors.
OMIGA~\cite{omiga/WangXZZ23} and CFCQL~\cite{cfcql/ShaoQCZJ23} are the latest offline MARL methods, both integrating value decomposition structures. 
OMIGA applies implicit local value regularization to enable in-sample learning, while CFCQL calculates counterfactual regularization per agent, avoiding the excessive conservatism caused by direct value decomposition-CQL integration. 
Nonetheless, the IGM principle has been shown to fail in identifying optimal policies in multi-modal reward landscapes~\cite{whyNoIGM/FuYXYW22}, due to the limited expressiveness of the Q-value network, which poses a potential risk of encountering the OOD joint actions issue in offline settings.

An alternative research direction in offline RL applies constraints on state-action distributions, called DIstribution Correction Estimation (DICE) methods~\cite{offlineSurvey/abs-2203-01387}.
AlberDICE~\cite{alberdice/MatsunagaLYLAK23} is a pioneering DICE-based method in offline MARL, which is proved to converge to NEs. However, when multiple NEs exist, its convergence results heavily depends on the dataset distribution. If the behavior policy is near a sub-optimal NE, AlberDICE will converge directly to that sub-optimal solution rather than the global optimum. This is primarily because AlberDICE lacks sufficient exploration of low-probability state-action pairs in dataset, leading to premature convergence to a deterministic policy.
Additionally, AlberDICE employs an out-of-sample learning during policy extraction, i.e., it uses actions produced by the policy rather the actions in datasets, which could lead to OOD joint actions~\cite{IVR/Xu0LYWCZ23, IQL/KostrikovNL22}.

Additionally, some works consider using model-based method~\cite{ModelBasedPPO} or using diffusion models~\cite{diffusion1, madiff} to solve the OOD action issue. 
More discussion about the related work is given in Appendix E.


\section{Background}

\subsection{Cooperative Markov Game}

The cooperative MARL problem is usually modeled as a cooperative Markov game~\cite{MarkovGame/Littman94} $\mathcal{G}=\langle{\mathcal{N}, \mathcal{S},\mathcal{A},P,r,\gamma, d}\rangle$, where
$\mathcal{N}=\{1,\cdots,N\}$ is the set of agent indices, $\mathcal{S}$ is the finite state space, $\mathcal{A}=\prod_{i\in\mathcal{N}} \mathcal{A}^i$ is the joint action space, with $\mathcal{A}_i$ denoting the finite action space of agent $i$, $r: \mathcal{S}\times\mathcal{A}\rightarrow\mathbb{R}$ is the common reward function shared with all agents, $P: \mathcal{S}\times\mathcal{A}\times\mathcal{S}\rightarrow[0, 1]$ is the transition probability function, $\gamma\in[0,1)$ is the discount factor, and $d\in\Delta(\mathcal{S})$ is the initial state distribution. 
At time step $t\in\{1,\cdots,T\}$, each agent $i\in\mathcal{N}$ at state $s_t\in\mathcal{S}$ selects an action $a_t^i\sim\pi^i(\cdot|s_t)$ and moves to the next state $s_{t+1} \sim P(\cdot|s_t, \boldsymbol{a}_t)$. It then receives a reward $r_t=r(s_t,\boldsymbol{a}_t)$ according to the joint action $\boldsymbol{a}_t=\{a_t^1,\cdots,a_t^N\}$. We denote the joint policy as $\boldsymbol{\pi}(\cdot|s)=\prod_{i\in\mathcal{N}}\pi^i(\cdot|s)$, and the joint policy  except the $i$-th player as $\boldsymbol{\pi}^{-i}$. 
In a cooperative Markov game, all agents aim to learn a optimal joint policy $\boldsymbol{\pi}$ that jointly maximizes the expected discount returns $\mathbb{E}_{s\in\mathcal{S}, \boldsymbol{a}\sim\boldsymbol{\pi}}[\sum_{t=0}^{T}\gamma^{t} r(s_t,\boldsymbol{a}_t)]$. 
Under the offline setting, only a pre-collected dataset $\mathcal{D}=\{(s,\boldsymbol{a},r,s^\prime)_{k}\}_{k=1}^{|\mathcal{D}|}$ collected by an unknown behavior policy $\boldsymbol{\mu}=\prod_{i\in\mathcal{N}}\mu^i$ is given and the environment interactions are not allowed.

\subsection{IGM Principle and Value Decomposition}

Value-based methods aim to learn a joint Q-function $\boldsymbol{Q}: \mathcal{S}\times\mathcal{A}\rightarrow\mathbb{R}$ to estimate the future expected return given the current state $s$ and joint action $\boldsymbol{a}$. However, directly computing the joint Q-function is challenging due to the huge state-action space in MARL. 
To address this issue, value decomposition decomposes the joint Q-function $\boldsymbol{Q}$ into individual Q-functions $Q^i$ for each agent: $\boldsymbol{Q}(s,\boldsymbol{a}) = f_\text{mix}(Q^1(s,a^1), \cdots, Q^N(s,a^N); s)$, where $f_\text{mix}$ represents the mixing function conditioned on the state~\cite{whyNoIGM/FuYXYW22}. 
The mixing function $f_\text{mix}$ must satisfy the IGM principle that any optimal joint action $\boldsymbol{a}_*$ should satisfy
\begin{equation}\label{eq: IGM}
    \boldsymbol{a}_*=\mathop{\arg\max}_{\boldsymbol{a}\in\mathcal{A}}\boldsymbol{Q}(s,\boldsymbol{a}) = \bigcup_{i\in\mathcal{N}}\{ \mathop{\arg\max}_{a^i\in\mathcal{A}^i} Q^i(s,a^i) \}.
\end{equation}
Under the IGM assumption, value decomposition enables the identification of the optimal joint action through the greedy actions of each agent.

\subsection{Behavior-Regularized Markov Game in Offline MARL}

Behavior-Regularized Markov Game is a useful framework to avoid distribution shift by incorporating a data-related regularization term on rewards~\cite{omiga/WangXZZ23, IVR/Xu0LYWCZ23}. In this framework, the goal is to optimize policy by
\begin{align}\label{eq: Behavior-Regularized MG}
    \max_{\boldsymbol{\pi}} \mathbb{E}\Big[ \sum_{t=1}^{T} \gamma^{t}\Big( r(s_t, \boldsymbol{a}_t) 
    &- \alpha f(\boldsymbol{\pi}(\cdot|s_t), \boldsymbol{\mu}(\cdot|s_t)) \Big)\Big],
\end{align}
where $f(\cdot, \cdot)$ is a regularization function, and $\alpha \ge 0$ is a temperature constant. The unknown behavior policy $\boldsymbol{\mu}$ here can usually be approximated by using Behavior Cloning~\cite{omiga/WangXZZ23, IVR/Xu0LYWCZ23}.
Policy evaluation operator in this framework is given by
\begin{align}\label{eq: policy evaluation in Behavior regularized MG}
    &\mathcal{T}_{\boldsymbol{\pi}} \boldsymbol{Q}_{\boldsymbol{\pi}}(s, \boldsymbol{a}) \triangleq r(s, \boldsymbol{a}) + \gamma \mathbb{E}_{s'|s,\boldsymbol{a}}\big[ V_{\boldsymbol{\pi}}(s') \big],
    \\\text{where }
    &V_{\boldsymbol{\pi}}(s) = \mathbb{E}_{\boldsymbol{a}\sim\boldsymbol{\pi}} \Big[ \boldsymbol{Q}_{\boldsymbol{\pi}}(s, \boldsymbol{a}) 
    - \alpha f(\boldsymbol{\pi}(\boldsymbol{a}|s), \boldsymbol{\mu}(\boldsymbol{a}|s))  \Big]. 
    \nonumber
\end{align}
Thus, the objective (\ref{eq: Behavior-Regularized MG}) can be represented as
\begin{align}\label{eq: objective in Behavior-Regularized MG}
    \max_{\boldsymbol{\pi}} \mathbb{E}_{\boldsymbol{a}\sim\boldsymbol{\pi}}\Big[ 
    \boldsymbol{Q}_{\boldsymbol{\pi}}(s, \boldsymbol{a}) 
    - \alpha f(\boldsymbol{\pi}(\cdot|s), \boldsymbol{\mu}(\cdot|s)) \Big].
\end{align}

\section{The Proposed Method}

\subsection{OOD Joint Action in Offline MARL}

\begin{figure}[t]
    \centering
    \includegraphics[width=0.65\linewidth]{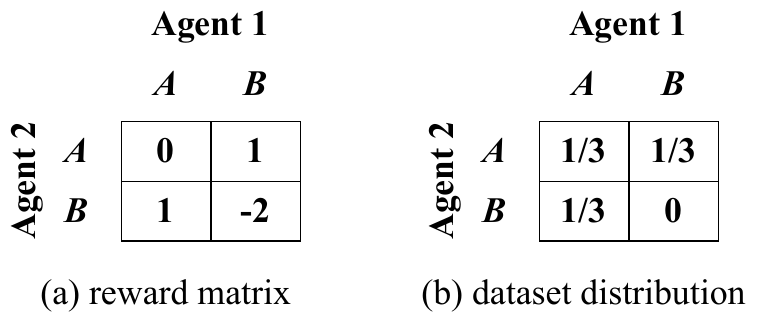}
    \caption{XOR game. (a) is the reward matrix of joint actions. (b) is the distribution of dataset. }
    \label{fig: xor game}
\end{figure}

In offline MARL, value decomposition methods are more prone to encountering OOD joint actions  due to the constraints of the IGM principle in certain scenarios. We use the XOR game, shown in Figure~\ref{fig: xor game}, to illustrate this phenomenon. Figure~\ref{fig: xor game}(a) shows the reward matrix of the XOR game, while Figure~\ref{fig: xor game}(b) depicts the dataset considered in the offline setting.
Since it is necessary to minimize temporal difference (TD) error $\mathbb{E}_{\mathcal{D}}[(f_\text{mix}(Q^1(a^1), Q^2(a^2)) - r(a^1, a^2))^2]$ while satisfying the IGM principle, the local Q-functions for both agents are forced to satisfy $Q^i(B) > Q^i(A),i=1,2$ (See Appendix D for a detailed derivation). As a result, both agents tend to choose action $B$, resulting in the OOD joint action $(B, B)$.

Another line in offline MARL research combines MAPG methods and data-related regularization~\cite{OMAR/PanH0X22}. 
However, they can still encounter the OOD joint actions issue although not constrained by the IGM principle. 
Considering again the above offline task, both learned agents are likely to choose $(A, A)$ due to the data-related regularization. For agent 1, given that its teammate selects action $A$, choosing action $B$ would yield a higher payoff. The same is true for agent 2, resulting in the OOD joint action $(B, B)$.
This situation arises because these methods do not fully consider the change of teammates' policies, leading to conflicting directions in policy updates.

MAPG methods employing sequential update scheme can effectively address this issue, as they fully consider the direction of teammates’ policy updates, thereby avoiding conflicts~\cite{alberdice/MatsunagaLYLAK23, HAPPO/KubaCWWSW022}.
In the same scenario as above, but with sequential updates, where agent 1 updates first followed by agent 2, agent 1 would still choose action B for a higher payoff. Then, when agent 2 updates, knowing that agent 1 chose B, it would find that sticking with action A is best. Consequently, sequential-update MAPG methods converge to the optimal policy.

\subsection{In-Sample Sequential Policy Optimization}

Inspired by the above discussions, we introduce an in-sample sequential policy optimization method under the behavior-regularized Markov game framework, i.e., Eq.(\ref{eq: objective in Behavior-Regularized MG}). 
Here we consider the reverse KL divergence as the regularization, which means $f(x,y)=\log(\frac{x}{y})$. 
The benefit of choosing reverse KL divergence is that the global regularization can be decomposed naturally as $\log(\frac{\boldsymbol{\pi}}{\boldsymbol{\mu}}) = \sum_{i\in\mathcal{N}} \log (\frac{\pi^i}{\mu^i})$, making the simplified computation of sequential-update possible. 
Denoting $i_{1:n}$ as an ordered subset $\{i_1,\cdots,i_n\}$ of $\mathcal{N}$, and $-i_{1:n}$ as its complement, where $i_k$ is the $k$-th agent in the ordered subset and $i_{1:0}=\emptyset$, the sequential-update objectives are given by:
\begin{align}\label{eq: sequential-update objective 1}
    \pi_\text{new}^{i_n}={\arg\max}_{\pi^{i_n}} \mathbb{E}_{a^{i_n}\sim\pi^{i_n}}\Big[ &
    Q^{i_{1:n}}_{\boldsymbol{\pi}_\text{old}}(s, a^{i_n}) 
    \nonumber\\&- 
    \alpha \log(\frac{\pi^{i_n}(a^{i_n}|s)}{\mu^{i_n}(a^{i_n}|s)} ) \Big],
\end{align}
 where
\begin{align*}
    Q^{i_{1:n}}_{\boldsymbol{\pi}_\text{old}}(s, a^{i_n})
    \triangleq
    \mathbb{E}_{\boldsymbol{\pi}_\text{new}^{i_{1:n-1}},\boldsymbol{\pi}_\text{old}^{-i_{1:n}}}\Big[ \boldsymbol{Q}_{\boldsymbol{\pi_\text{old}}}(s, \boldsymbol{a}^{-i_{n}}, a^{i_n}) \Big]. 
\end{align*}
However, the optimization objective (\ref{eq: sequential-update objective 1}) requires actions produced by the policy, which is in a out-of-sample learning manner, potentially leading to OOD actions. 
In order to achieve in-sample learning using only the dataset actions, we derive the closed-form solution of objectives (\ref{eq: sequential-update objective 1}) by the Karush-Kuhn-Tucker (KKT) conditions
\begin{align}
    \pi_\text{new}^{i_n}(a^{i_n}|s) \propto
    \mu^{i_n}(a^{i_n}|s) \cdot 
    \exp\Big(
    \frac{Q^{i_{1:n}}_{\boldsymbol{\pi}_\text{old}}(s, a^{i_n})}{\alpha}
    \Big),
\end{align}
and thus obtain the in-sample optimization objectives for parametric policy $\pi_{\theta^{i_n}}$ by minimizing the KL divergence:
\begin{align}\label{eq: in-sample 1}
    \theta^{i_n}_\text{new} = \mathop{\arg\min}_{\theta^{i_n}}&
    D_\text{KL}(\pi_\text{new}^{i_n}(\cdot|s), \pi_{\theta^{i_n}}(\cdot|s))
    \nonumber\\
    =\mathop{\arg\min}_{\theta^{i_n}}&
    \mathbb{E}_{(s,a^{i_n})\sim\mathcal{D}}\Big[
    -\exp\Big(
    \frac{A^{i_{1:n}}_{\boldsymbol{\pi}_\text{old}}(s, a^{i_n})}{\alpha}
    \Big)
    \nonumber\\&\cdot
    \log \pi_{\theta^{i_n}}(a^{i_n}|s)
    \Big],
\end{align}
where $A^{i_{1:n}}_{\boldsymbol{\pi}_\text{old}}(s, a^{i_n})
    \triangleq
    Q^{i_{1:n}}_{\boldsymbol{\pi}_\text{old}}(s, a^{i_n}) - 
    \mathbb{E}_{\pi^{i_n}_\text{new}} [ Q^{i_{1:n}}_{\boldsymbol{\pi}_\text{old}}(s, a^{i_n}) ]$. 

A potential problem of this method is that it may lead to premature convergence to local optima due to exploitation of vested interests. This concern is especially pronounced when the behavior policy is a local optimum, as we will show in the next subsection.

\subsection{Maximum-Entropy Behavior-Regularized Markov Game}

\begin{figure}[t]
    \centering
    \includegraphics[width=0.7\linewidth]{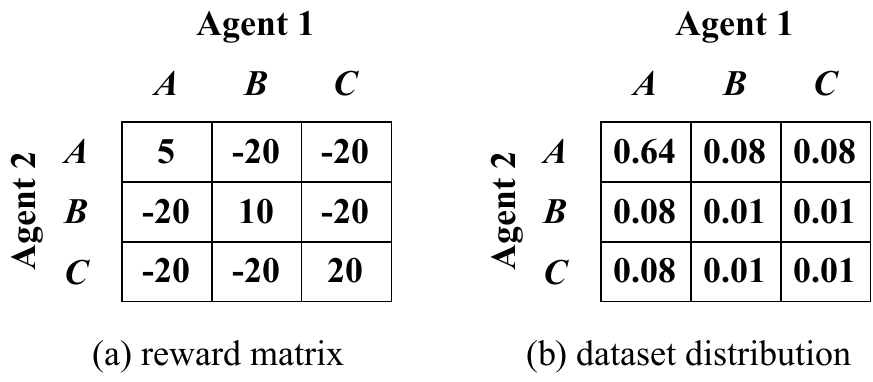}
    \caption{M-NE game. (a) is the reward matrix of joint actions. (b) is the distribution of dataset. }
    \label{fig: M-NE game}
\end{figure}

The existence of multiple local optima is a common phenomenon in many multi-agent tasks, where finding the global optimum is often extremely challenging. Therefore, near-optimal (or expert) behavior policies can easily fall into or stay near local optima.
In such cases, because the data-related regularization enforces the learned policy to remain close to the behavior policy, optimizing the objective in Eq.(\ref{eq: sequential-update objective 1}) is more likely to cause the sequential policy optimization method to converge towards a deterministic policy that exploits this local optimum. 
Moreover, escaping this local optimum becomes challenging, as when one of the agents attempts to deviate unilaterally, the optimization objective~(\ref{eq: sequential-update objective 1}) impedes this since it hurts the overall benefits.

We examine this issue using the M-NE game depicted in Figure~\ref{fig: M-NE game}, with Figure~\ref{fig: M-NE game}(a) showing the reward matrix and Figure~\ref{fig: M-NE game}(b) illustrating the offline dataset.
In this game, there are three NEs: $(A, A)$, $(B, B)$, and $(C, C)$, with rewards of 5, 10, and 20, respectively, where $(C, C)$ represents the global optimal NE and other NEs are local optima. On the considered dataset, data-related regularization enforces agents to select $A$ with a high probability. As a result, agents confidently converge to the local optimum $(A, A)$ based on the observed high probability of their teammates choosing $A$, failing to recognize the optimal joint action $(C, C)$.

One way to address this issue is to introduce perturbations to the rewards, preventing sequential policy optimization method from deterministically converging to a local optimum and thereby encouraging it to escape the local optimum and identify the global optimal solution. 
From a game-theoretic perspective, the optimal solution of the perturbed game aligns with the solution concept of quantal response equilibrium (QRE)~\cite{QRE/mckelvey1995quantal}. 
\begin{definition}\label{def: QRE}
    For a Behavior-Regularized Markov Game $\mathcal{G}$ with a reward function $r$, denote the perturbed reward as $\tilde{r}$. 
    Then, a joint policy $\boldsymbol{\pi}_*$ is a QRE if it holds
    \begin{align}
        \tilde{J}(\boldsymbol{\pi}_*)\ge \tilde{J}(\boldsymbol{\pi}_*^{-i}, \pi^i), 
        \quad \forall{i}\in\mathcal{N}, \pi^i, 
    \end{align}
    where $\tilde{J}(\boldsymbol{\pi})\triangleq \mathbb{E}_{\boldsymbol{\pi}}[ \sum_{t} \gamma^{t}( \tilde{r}(s_t, \boldsymbol{a}_t) 
    - \alpha f(\boldsymbol{\pi}(\cdot|s_t), \boldsymbol{\mu}(\cdot|s_t)) )]$.
\end{definition}
Therefore, our goal is to design an in-sample sequential policy optimization method with QRE convergence guarantees. 
One simple and effective way to introduce disturbances is to add policy entropy into the rewards, which is also a commonly used regularization in online RL to improve exploration~\cite{HASAC/LiuZHFFC024, sac}.
Therefore, we introduce the following Maximum-Entropy Behavior-Regularized Markov Game (MEBR-MG) problem, which is a generalization of Behavior-Regularized Markov Game~(\ref{eq: Behavior-Regularized MG}). 
\begin{align}\label{eq: MEBR-MG}
    \max_{\boldsymbol{\pi}} \mathbb{E}\Big[ \sum_{t=1}^{T} \gamma^{t}\Big( r(s_t, \boldsymbol{a}_t) 
    &- \alpha D_\text{KL}(\boldsymbol{\pi}(\cdot|s_t), \boldsymbol{\mu}(\cdot|s_t))
    \nonumber
    \\&+ 
    \beta \mathcal{H}(\boldsymbol{\pi}(\cdot|s_t)) \Big)\Big],
\end{align}
where $\mathcal{H}(\boldsymbol{\pi}(\cdot|s_t))$ is policy entropy, and $\beta\ge0$ is a temperature constant. 
In the following context, we first give some facts about MEBR-MG, and then give the in-sample sequential policy optimization method under MEBR-MG. 

In MEBR-MG, we have the following modified policy evaluation operator given by:
\begin{align}\label{eq: policy evaluation operator}
    &\mathcal{T}_{\boldsymbol{\pi}} \boldsymbol{Q}_{\boldsymbol{\pi}}(s, \boldsymbol{a}) \triangleq r(s, \boldsymbol{a}) + \gamma \mathbb{E}_{s'|s,\boldsymbol{a}}\big[ V_{\boldsymbol{\pi}}(s') \big],
\end{align}
where
\begin{align*}
    V_{\boldsymbol{\pi}}(s) =& \mathbb{E}_{\boldsymbol{a}\sim\boldsymbol{\pi}} \Big[ \boldsymbol{Q}_{\boldsymbol{\pi}}(s, \boldsymbol{a}) 
    \nonumber
    \\&- \sum_{i\in\mathcal{N}}\Big(\alpha \log \frac{{\pi}^{i}(a^i|s)}{\mu^i(a^i|s)}  
    +\beta \log {\pi}^i(a^i|s)\Big) \Big]. 
\end{align*}

\begin{lemma}\label{lemma: policy evaluation}
    Given a policy $\boldsymbol{\pi}$, consider the modified policy evaluation operator $\mathcal{T}_{\boldsymbol{\pi}}$ in Eq.(\ref{eq: policy evaluation operator}) and a initial Q-function $\boldsymbol{Q}_0: \mathcal{S}\times\mathcal{A}\rightarrow\mathbb{R}$, and define $\boldsymbol{Q}_{k+1}=\mathcal{T}_{\boldsymbol{\pi}} \boldsymbol{Q}_{k}$. 
    Then the sequence $\boldsymbol{Q}_{k}$ will converge to the Q-function $\boldsymbol{Q}_{\boldsymbol{\pi}}$ of policy $\boldsymbol{\pi}$ as $k\rightarrow\infty$. 
\end{lemma}
Proof can be found in Appendix A.
This lemma indicates Q-function will converge to the Q-value under the joint policy $\boldsymbol{\pi}$ by repeatedly applying the policy evaluation operator. 

Moreover, the additional smoothness introduced by the regularization term allows the QRE of MEBR-MG to be expressed in the form of Boltzmann distributions, as demonstrated by the following Proposition~\ref{prop: representation of optimal policy}.
\begin{proposition}\label{prop: representation of optimal policy}
    In a MEBR-MG, a joint policy $\boldsymbol{\pi}_*$ is a QRE if it holds
    \begin{align}
        V_{\boldsymbol{\pi}_{*}}(s) \ge V_{\pi^i, \boldsymbol{\pi}_{*}^{-i}}(s), 
        \quad \forall{i}\in\mathcal{N}, \pi^i, s \in \mathcal{S}. 
    \end{align}
    Then the QRE policies for each agent $i$ are given by
    \begin{align}\label{eq: QRE}
        \pi^{i}_*(a^i|s) &\propto 
        \mu^i(a^i|s) 
        \nonumber\\
        \cdot\exp &\Big( \frac{\mathbb{E}_{\boldsymbol{a}^{-i}\sim\boldsymbol{\pi}^{-i}_{*}}[\boldsymbol{Q}_{\boldsymbol{\pi}_*}(s, a^i, \boldsymbol{a}^{-i})] - \beta \log \mu^i(a^i|s)}{\alpha + \beta} \Big)
    \end{align}
\end{proposition}
Proof can be found in Appendix A.
Eq.(\ref{eq: QRE}) for QRE demonstrates that incorporating the reverse KL divergence term ensures that the learned policy shares the same support set as the behavior policy, thereby avoiding OOD actions; and the addition of an entropy term allows the policy to place greater emphasis on actions with lower probabilities in the behavior policy, preventing premature convergence to local optima.

\begin{algorithm}[tb]
\caption{InSPO}
\label{alg: InSPO}
\textbf{Input}: Offline dataset $\mathcal{D}$, initial policy $\boldsymbol{\pi}_0$ and Q-function $\boldsymbol{Q}_0$\\
\textbf{Output}: $\boldsymbol{\pi}_{K}$
\begin{algorithmic}[1] 
\STATE Compute behavior policy $\boldsymbol{\mu}$ by simple Behavior Cloning
\FOR{$k = 1, \cdots, K$}
    \STATE Compute $\boldsymbol{Q}_{k}$ by Eq.(\ref{eq: policy evaluation operator})
    \STATE Draw a permutation $i_{1:N}$ of agents at random
    \FOR{$n=1,\cdots,N$}
        \STATE Update $\boldsymbol{\pi}_{k}^{i_n}$ by Eq.(\ref{eq: in-sample 2})
    \ENDFOR
\ENDFOR
\end{algorithmic}
\end{algorithm}

Similar to Eq.(\ref{eq: in-sample 1}), the in-sample sequential policy optimization procedure under MEBR-MG is given by:
\begin{align}\label{eq: in-sample 2}
    &\theta^{i_n}_\text{new} 
    =\mathop{\arg\min}_{\theta^{i_n}}
    \mathbb{E}_{(s,a^{i_n})\sim\mathcal{D}}\Big[
    \nonumber\\&
    -\exp\Big(
    \frac{A^{i_{1:n}}_{\boldsymbol{\pi}_\text{old}}(s, a^{i_n}) - \beta \log \mu^{i_n}(a^{i_n}|s)}{\alpha+\beta}
    \Big)
    \cdot
    \log \pi_{\theta^{i_n}}(a^{i_n}|s)
    \Big]. 
\end{align}

\begin{proposition}\label{prop: policy improvement}
    The sequential policy optimization procedure under MEBR-MG guarantees policy improvement, i.e., $\forall{s}\in\mathcal{S}, a\in\mathcal{A}$, 
    \begin{align*}
        &\boldsymbol{Q}_{\boldsymbol{\pi}_\text{new}}(s, \boldsymbol{a})
        \ge
        \boldsymbol{Q}_{\boldsymbol{\pi}_\text{old}}(s, \boldsymbol{a}) 
        ,
        V_{\boldsymbol{\pi}_\text{new}}(s) \ge V_{\boldsymbol{\pi}_\text{old}}(s). 
    \end{align*}
\end{proposition}
Proof can be found in Appendix A. 
Proposition~\ref{prop: policy improvement} demonstrates that the policy improvement step defined in Eq.(\ref{eq: in-sample 2}) ensures a monotonic increase in performance at each iteration. By alternating between the policy evaluation step and the policy improvement step, we derive InSPO, as shown in Algorithm~\ref{alg: InSPO}, and we furthermore prove that InSPO converges to QRE as follows. 
\begin{theorem}\label{thrm: converge to QRE}
    Joint policy $\boldsymbol{\pi}$ updated by Algorithm~\ref{alg: InSPO} converges to QRE. 
\end{theorem}
Proof can be found in Appendix A.

\begin{table*}[ht!]
    \small
    \centering
    \begin{tabular}{c|c c c c c c}
         Dataset
         &BC
         &OMAR
         &AlberDICE
         &CFCQL
         &OMIGA
         &InSPO
         \\
         \hline
         (a)
         &$0.00\pm 0.01$
         &$\mathbf{1.00\pm 0.00}$
         &$\mathbf{1.00\pm 0.00}$
         &$-0.64\pm 0.71$
         &$0.00\pm 0.01$
         &$\mathbf{1.00\pm 0.00}$
         \\
         (b)
         &$0.23\pm 0.01$
         &$-2.00\pm 0.00$
         &$\mathbf{1.00\pm 0.00}$
         &$-0.38\pm 0.11$
         &$0.21\pm 0.03$
         &$\mathbf{1.00\pm 0.00}$
         \\
         (c)
         &$0.00\pm 0.01$
         &$0.00\pm 0.00$
         &$\mathbf{1.00\pm 0.00}$
         &$-0.73\pm 0.48$
         &$0.05\pm 0.00$
         &$\mathbf{1.00\pm 0.00}$
    \end{tabular}
    \caption{Averaged test return on XOR game. }
    \label{tab: xor game}
\end{table*}

\begin{table*}[ht!]
    \small
    \setlength{\tabcolsep}{1mm}
    \centering
    \begin{tabular}{c|c c c c c c}
         Dataset
         &BC
         &OMAR
         &AlberDICE
         &CFCQL
         &OMIGA
         &InSPO
         \\
         \hline
         balanced
         &$-9.79\pm 0.41$
         &$\mathbf{20.00\pm 0.00}$
         &$\mathbf{20.00\pm 0.00}$
         &$\mathbf{20.00\pm 0.00}$
         &$\mathbf{20.00\pm 0.00}$
         &$\mathbf{20.00\pm 0.00}$
         \\
         imbalanced
         &$-3.47 \pm 0.17$
         &$5.00\pm 0.00$
         &$5.00\pm 0.00$
         &$5.00\pm 0.00$
         &$5.00\pm 0.00$
         &$\mathbf{20.00\pm 0.00}$
    \end{tabular}
    \caption{Averaged test return on M-NE game. }
    \label{tab: m-ne game}
\end{table*}

\subsection{The Practical Implementation of InSPO}

In this section, we design a practical implementation of InSPO to handle the issue of large state-action space, making it more suitable for offline MARL. 
More details can be found in Appendix B.

\paragraph{Policy Evaluation. }
According to Eq.(\ref{eq: policy evaluation operator}), we need to train a global Q-function to estimate the expected future return based on the current state and joint action. However, in MARL, the joint action space grows exponentially with the number of agents. To circumvent this exponential complexity, we instead maintain a local Q-function $Q_{\phi^{i_n}}$ for each agent $i_n\in\mathcal{N}$ to approximate $Q^{i_{1:n}}_{\boldsymbol{\pi}_\text{old}}(s, a^{i_n})$.  
Besides, $Q_{\phi^{i_n}}$ should be updated sequentially in conjunction with the policy in order to incorporate the information of updated teammates (i.e., ${\boldsymbol{\pi}_\text{new}^{i_{1:n-1}}}$) into the local Q-function. 
Thus, we optimize the following objective for each local Q-function $\phi^{i_n}$:
\begin{align}\label{eq: update of Q}
    \min_{\phi^{i_n}} \mathbb{E}_{(s,\boldsymbol{a}, s^\prime, r)\sim\mathcal{D}}
    \Big[\rho^{i_n}\cdot
    \Big(
    Q_{\phi^{i_n}}(s, a^{i_n}) - 
    y
    \Big)^2\Big], 
\end{align}
where
\begin{align*}
    \rho^{i_n} &= \rho^{i_n}(s,\boldsymbol{a}) \triangleq \frac{(\boldsymbol{\pi}_\text{new}^{i_{1:n-1}}\cdot \boldsymbol{\pi}_\text{old}^{-i_{1:n}})(\boldsymbol{a}^{-i_n}|s)}{\boldsymbol{\mu}^{-i_n}( \boldsymbol{a}^{-i_n}|s)} ,
    \\
    y =& y(s,\boldsymbol{a}, s^\prime, r) \triangleq r + \gamma \mathbb{E}_{a^{i_{n}\prime}\sim\pi_\text{old}^{i_n}} \big[Q_{{\phi}^{i_n}}(s^{\prime}, a^{i_n \prime}) \big]
    \\&-
    \alpha D_\text{KL}(\pi_\text{old}^{i_{n}}(\cdot | s^\prime), \mu^{i_n}(\cdot | s^\prime)) + \beta \mathcal{H}(\pi_\text{old}^{i_{n}}(\cdot | s^\prime)). 
\end{align*}
Here we omit the regularization terms for other agents to simplify the computation. 
Furthermore, to reduce the high variance of importance sampling ratio $\rho^{i_n}$, InSPO adopts importance resampling~\cite{resampling/SchlegelCGQW19} in practice, which resamples experience with probability proportional to $\rho^{i_n}$ to construct a resampled dataset $\mathcal{D}_{\rho^{i_n}}$, stabilizing the  algorithm training effectively. Thus, Eq.(\ref{eq: update of Q}) is replaced with 
\begin{align}\label{eq: resample update of Q}
    \min_{\phi^{i_n}} \mathbb{E}_{(s,\boldsymbol{a}, s^\prime, r)\sim\mathcal{D}_{\rho^{i_n}}}
    \Big[
    \Big(
    Q_{\phi^{i_n}}(s, a^{i_n}) - 
    y
    \Big)^2\Big].
\end{align}

\paragraph{Policy Improvement. }

After obtaining the optimal local value functions, we can adopt the in-sample sequential policy optimization method in Eq.(\ref{eq: in-sample 2}) to learn the local policy for each agent:
\begin{align*}\label{eq: final in-sample}
    &\theta^{i_n}_\text{new} 
    =\mathop{\arg\min}_{\theta^{i_n}}
    \mathbb{E}_{(s,a^{i_n})\sim\mathcal{D}_{\rho^{i_n}}}\Big[
    \nonumber\\&
    -\exp\Big(
    \frac{A_{\phi^{i_n}}(s, a^{i_n}) - \beta \log \mu^{i_n}(a^{i_n}|s)}{\alpha+\beta}
    \Big)
    \log \pi_{\theta^{i_n}}(a^{i_n}|s)
    \Big],
\end{align*}
where $A_{\phi^{i_n}}(s, a^{i_n})
    \triangleq
    Q_{\phi^{i_n}}(s, a^{i_n}) - 
    \mathbb{E}_{\pi_{\theta^{i_n}_\text{old}}} [ Q_{\phi^{i_n}}(s, a^{i_n}) ]$.

\begin{table*}[ht!]
    \small
    \setlength{\tabcolsep}{1mm}
    \centering
    \begin{tabular}{c c|c c c c c c}
         &Dataset
         &BC
         &OMAR
         &AlberDICE
         &CFCQL
         &OMIGA
         &InSPO
         \\
         \hline
         Optimal
         &$-1.26$
         &$-2.21\pm 0.90$
         &$-6.01\pm 0.00$
         &$-1.27\pm 0.03$*
         &$-12.75\pm 1.60$
         &$-9.87\pm 0.68$
         &$\mathbf{-1.26\pm 0.00}$
         \\
         Mixed
         &$-4.56$
         &$-5.88\pm 0.49$
         &$-6.01\pm 0.00$
         &$-1.29\pm 0.00$
         &$-13.79\pm 1.78$
         &$-13.45\pm 0.42$
         &$\mathbf{-1.27\pm 0.01}$
    \end{tabular}
    \caption{Averaged test return on Bridge. }
    \label{tab: Bridge}
\end{table*}
\begin{table*}[ht!]
    \small
    \centering
    \begin{tabular}{c c|c c c c c c}
         Map
         &Dataset
         &BC
         &OMAR
         &CFCQL
         &OMIGA
         &InSPO
         \\
         \hline
         \multirow{4}*{2s3z}
         &medium
         &$0.16\pm 0.07$
         &$0.15\pm 0.04$
         &$\mathbf{0.40\pm0.10}$
         &$0.23\pm 0.01$
         &$0.23\pm 0.06$
         \\
         ~
         &medium-replay
         &$0.33\pm 0.04$
         &$0.24\pm 0.09$
         &$0.55\pm 0.07$*
         &$0.42\pm 0.02$
         &$\mathbf{0.58\pm0.09}$
         \\
         ~
         &expert
         &$0.97\pm0.02$
         &$0.95\pm0.04$
         &$\mathbf{0.99\pm0.01}$
         &$0.98\pm 0.02$
         &$\mathbf{0.99\pm0.01}$
         \\
         ~
         &mixed
         &$0.44\pm0.06$
         &$0.60\pm0.04$
         &$0.84\pm0.09$*
         &$0.62\pm 0.03$
         &$\mathbf{0.85\pm0.04}$
         \\
         \hline
         \multirow{4}*{3s\_vs\_5z}
         &medium
         &$0.08\pm0.02$
         &$0.00\pm0.00$
         &$\mathbf{0.28\pm0.03}$
         &$0.02\pm 0.02$
         &$0.17\pm 0.05$
         \\
         ~
         &medium-replay
         &$0.01\pm0.01$
         &$0.00\pm0.00$
         &$\mathbf{0.12\pm0.04}$
         &$0.02\pm 0.01$
         &$0.10\pm 0.05$*
         \\
         ~
         &expert
         &$0.98\pm0.02$
         &$0.64\pm0.08 $
         &$\mathbf{0.99\pm0.01}$
         &$0.98\pm 0.02$
         &$\mathbf{0.99\pm0.01}$
         \\
         ~
         &mixed
         &$0.21\pm0.04$
         &$0.00\pm0.00$
         &$0.60\pm0.14$
         &$0.20\pm 0.06$
         &$\mathbf{0.78\pm0.09}$
         \\
         \hline
         \multirow{4}*{5m\_vs\_6m}
         &medium
         &$0.28\pm0.37$*
         &$0.19\pm 0.06$
         &$\mathbf{0.29\pm0.05}$
         &$0.25\pm 0.08$
         &$0.28\pm0.06$*
         \\
         ~
         &medium-replay
         &$0.18\pm0.06$
         &$0.03\pm 0.02 $
         &$0.22\pm0.06$*
         &$0.16\pm 0.05$
         &$\mathbf{0.24\pm0.07}$
         \\
         ~
         &expert
         &$0.82\pm 0.04$
         &$0.33\pm 0.06$
         &$\mathbf{0.84\pm 0.03}$
         &$0.74\pm 0.05$
         &$0.79\pm 0.12$
         \\
         ~
         &mixed
         &$0.21\pm 0.12$
         &$0.10\pm 0.10$
         &$0.76\pm 0.07$*
         &$0.38\pm 0.23$
         &$\mathbf{0.78\pm 0.06}$
         \\
         \hline
         \multirow{4}*{6h\_vs\_8z}
         &medium
         &$0.40\pm 0.03$
         &$0.04\pm 0.03$
         &$0.41\pm 0.04$*
         &$0.34\pm 0.01$
         &$\mathbf{0.43\pm 0.06}$
         \\
         ~
         &medium-replay
         &$0.11\pm 0.04$
         &$0.00\pm 0.00$
         &$0.21\pm 0.05$
         &$0.11\pm 0.04$
         &$\mathbf{0.23\pm 0.02}$
         \\
         ~
         &expert
         &$0.60\pm 0.04$
         &$0.01\pm 0.01$
         &$0.70\pm 0.06$*
         &$0.54\pm 0.04$
         &$\mathbf{0.74\pm 0.11}$
         \\
         ~
         &mixed
         &$0.27\pm0.06$
         &$0.00\pm 0.00$
         &$0.49\pm 0.08$
         &$0.36\pm 0.06$
         &$\mathbf{0.60\pm 0.12}$
         \\
         \hline
         ~
         &average performance
         &$0.38$
         &$0.21$
         &$0.54$
         &$0.39$
         &$\mathbf{0.55}$
    \end{tabular}
    \caption{Averaged test winning rate on StarCraft II Micromanagement. }
    \label{tab: sc2}
\end{table*}

\section{Experiments}

We conduct a series of experiments to evaluate InSPO on XOR game, M-NE game, Bridge~\cite{whyNoIGM/FuYXYW22} and StarCraft II Micromanagement~\cite{SMAC/Xu0LYWCZ23}. 
In addition to Behavior Cloning (BC), our baselines also include the current state-of-the-art offline MARL algorithms: OMAR~\cite{OMAR/PanH0X22}, CFCQL~\cite{cfcql/ShaoQCZJ23}, OMIGA~\cite{omiga/WangXZZ23} and AlberDICE~\cite{alberdice/MatsunagaLYLAK23}. 
Each algorithm is run for five random seeds, and we report the mean performance with standard deviation. 
For the final results, we indicate the algorithm with the best mean performance in bold, and an asterisk (*) denotes that the metric is not significantly different from the top-performing metric in that case, based on a heteroscedastic two-sided t-test with a 5\% significance level.
See Appendix C for experimental details.


\subsection{Comparative Evaluation}

\paragraph{Matrix Game. }

\begin{figure}[hb!]
    \centering
    \includegraphics[width=0.9\linewidth]{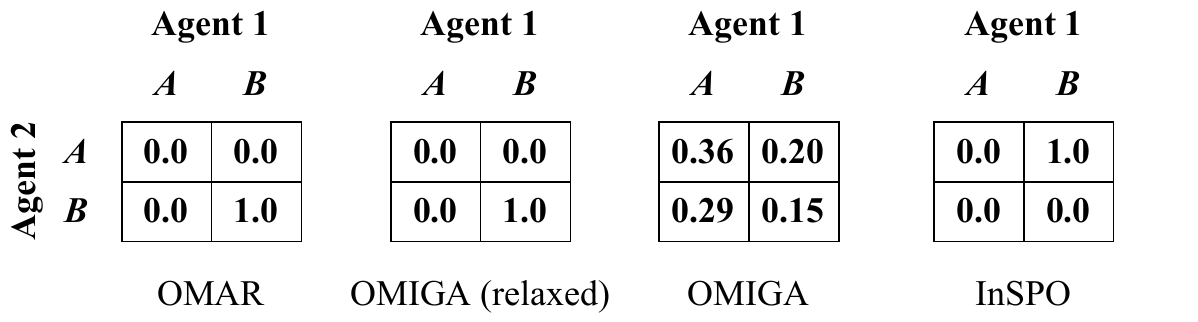}
    \caption{Final joint policy on XOR game for dataset (b). }
    \label{fig: converged policy on XOR}
\end{figure}
\begin{figure}[hb!]
    \centering
    \includegraphics[width=0.3\linewidth]{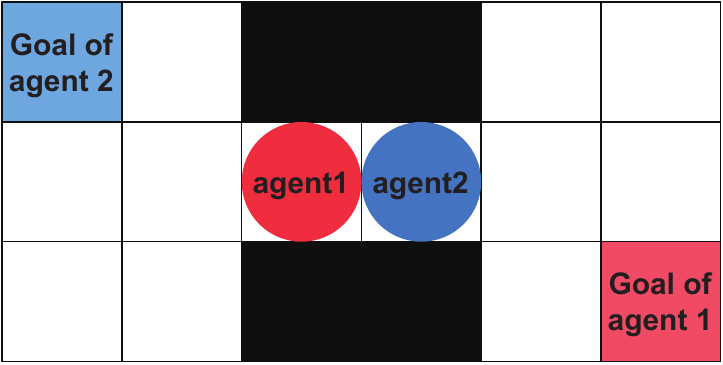}
    \caption{Bridge at the beginning. }
    \label{fig: Bridge}
\end{figure}

We evaluate whether InSPO can address the two issues highlighted in previous section using the XOR game and M-NE game shown in Figure~\ref{fig: xor game}(a) and Figure~\ref{fig: M-NE game}(a). 
First, we evaluate the ability of InSPO to handle the OOD joint actions issue using the XOR game. Table~\ref{tab: xor game} compares the performance of all algorithms on four datasets, each comprising an equal mix of different joint actions: (a) $\{(A,B), (B, A)\}$, (b) $\{(A, A), (A,B), (B, A)\}$, and (c) $\{(A, A), (A,B), (B, A), (B, B)\}$. Figure~\ref{fig: converged policy on XOR} illustrates the converged joint policy of OMAR, OMIGA, and InSPO on dataset (c), representing decentralized training, value decomposition, and sequential-update methods, respectively.

As observed, only InSPO and AlberDICE, two sequential-update MAPG methods, successfully converge to the optimal policy, while the other algorithms fail, even opting for OOD joint actions. 
Specifically, when a more relaxed behavior policy constraint is applied, the value decomposition method (i.e., OMIGA (relaxed) in Figure~\ref{fig: converged policy on XOR}) converges to an OOD joint action $(B, B)$, consistent with our analysis in previous section. These results suggest that decentralized training and value decomposition methods have limitations in environments that demand high levels of coordination. 

Next, we conduct InSPO on M-NE game to evaluate its ability to alleviate the local optimum convergence issue. Table~\ref{tab: m-ne game} shows the results on datasets: (a) a balanced dataset collected by a uniform policy $\mu^i(A)=\mu^i(B)=\mu^i(C)=1/3$, and (b) a imbalanced dataset collected by a near local optimum $\mu^i(A)=0.8, \mu^i(B)=\mu^i(C)=0.1$, for $i=1,2$. 
The results show that on the balanced dataset (a), most algorithms find the global optimal NE, while on the imbalanced dataset (b), only InSPO correctly identifies the global optimal NE. 
This indicates that in environments with multiple local optima, the convergence of most algorithms can be heavily influenced by the dataset distribution. 
Specifically, when the dataset is biased toward a local optimum, algorithms are prone to converging on this sub-optimal solution. In contrast, InSPO  converges to the global optimal solution through comprehensive exploration of the dataset. 
The results demonstrate in offline scenarios, algorithms must make full use of all available dataset information to prevent being heavily influenced by behavior policies.

\paragraph{Bridge. }

Bridge, illustrated in Figure~\ref{fig: Bridge}, is a grid-world Markov game resembling a temporal version of the XOR game. Two agents must alternately cross a one-person bridge as fast as possible. Starting side by side on the bridge, they must move together to allow one agent to cross first.
For this experiment, we use two datasets provided by \citet{alberdice/MatsunagaLYLAK23}: optimal and mixed. 
The optimal dataset contains 500 trajectories, generated by combining two optimal deterministic policies: either agent 1 steps back to let agent 2 cross first, or vice versa. The mixed dataset includes the optimal dataset plus an additional 500 trajectories generated using a uniform random policy.

The performance is shown in Table~\ref{tab: Bridge}, where the results of BC, OMAR and AlberDICE are from the report in \citet{alberdice/MatsunagaLYLAK23}. 
The performance is similar to that of the XOR game: only InSPO and AlberDICE, both using sequential-update, achieve near-optimal performance on both datasets. In contrast, both value decomposition methods fail to converge and produced undesirable outcomes.

\paragraph{StarCraft II. }

We further extend our study to the StarCraft II micromanagement benchmarks, a high-dimensional and complex environment widely used in both online and offline MARL. In this environment, we consider four representative maps: 2 easy maps (2s3z, 3s\_vs\_5z), 1 hard map (5m\_vs\_6m) and 1 super-hard map (6h\_vs\_8z). 
We use four datasets provided by \citet{cfcql/ShaoQCZJ23}: medium, expert, medium-replay and mixed. The medium-replay dataset is a replay buffer collected during training until the policy achieves medium performance, while the mixed dataset is the equal mixture of the medium and expert datasets. The results are shown in Table~\ref{tab: sc2}, with the performance of CFCQL, OMAR, and BC taken from \citeauthor{cfcql/ShaoQCZJ23}'s report.

In contrast to the previous benchmarks, StarCraft II does not exhibit a highly multi-modal reward landscape. Additionally, the agents share nearly identical local objectives, making this environment suitable for the IGM principle. Therefore, value decomposition methods have achieved state-of-the-art performance in this environment both in offline and online settings. Even so, as shown in Table~\ref{tab: sc2}, InSPO still demonstrates competitive performance and achieves state-of-the-art results in most tasks.

\begin{figure}[t!]
    \centering
    \includegraphics[width=0.65\linewidth]{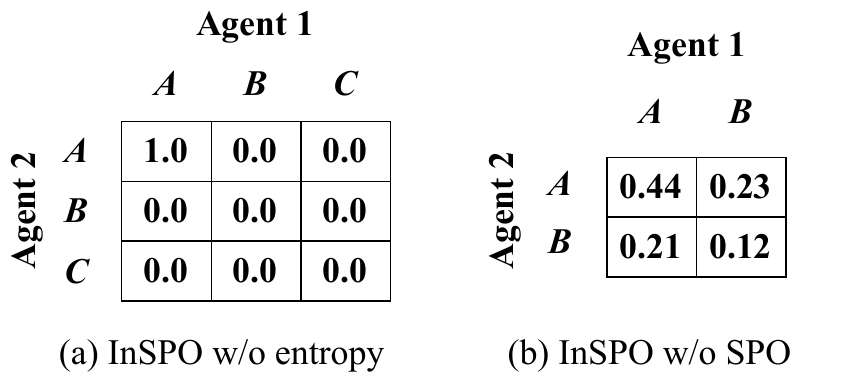}
    \caption{Ablation on entropy and sequential update scheme. (a) is InSPO without entropy on M-NE game for the imbalanced dataset. (b) is  simultaneous-update version of InSPO on XOR game for dataset (b). }
    \label{fig: ablation}
\end{figure}

\begin{table}[t!]
    \small
    \setlength{\tabcolsep}{0.45mm}
    \centering
    \begin{tabular}{c|c c c c|c}
         dataset
         &$\alpha=0.1$
         &$\alpha=5$
         &$\alpha=10$
         &$\alpha=50$
         &auto-$\alpha$
         \\
         \hline
         expert
         &$0.51\pm 0.17$
         &$0.54\pm 0.06$
         &$0.69\pm 0.05$
         &$0.62\pm 0.03$
         &$0.74$ 
         \\
         mixed
         &$0.17\pm 0.07$
         &$0.54\pm 0.06$
         &$0.53\pm 0.12$
         &$0.48\pm 0.07$
         &$0.60$ 
    \end{tabular}
    \caption{Ablation results for $\alpha$ on 6h\_vs\_8z. }
    \label{tab: ablation on alpha}
\end{table}

\paragraph{Ablation Study. }

Here we present the impact of different components on performance of InSPO. 
Figure~\ref{fig: ablation}(a) shows the converged policy of InSPO without entropy in the M-NE game on the imbalanced dataset. Without the perturbations of entropy in the optimization objective, InSPO w/o entropy cannot escape the local optimum. 
Figure~\ref{fig: ablation}(b) shows the policy of InSPO using the simultaneous update scheme instead of sequential policy optimization (denoted as InSPO w/o SPO) on dataset (b) of the XOR game. Due to conflicting update directions, InSPO w/o SPO fails to learn the optimal policy and faces the OOD joint actions issue.

Temperature $\alpha$ is used to control the degree of conservatism. A too large $\alpha$ will result in an overly conservative policy, while a too small one will easily causes distribution shift. 
Thus, to obtain a suitable $\alpha$, we implement both fixed and auto-tuned $\alpha$ in practice (see Appendix B for details), where the auto-tuned $\alpha$ is adjusted by
$\min_{\alpha} \mathbb{E}_{\mathcal{D}}[ \alpha D_\text{KL}( \boldsymbol{\pi}, \boldsymbol{\mu}) - \alpha \bar D_\text{KL}]$, where $\bar D_\text{KL}$ is the target value. 
Table~\ref{tab: ablation on alpha} gives ablation results for $\alpha$, which shows that the auto-tuned $\alpha$ can find an appropriate $\alpha$ to further improve performance. 

Furthermore, we explore the impact of update order on performance and the training efficiency of sequential updates. These results are provided in Appendix C.

\section{Conclusion}

In this paper, we study the offline MARL problem, a topic of significant practical importance and challenges that has not received adequate attention. 
We begin with two simple yet highly illustrative matrix games, highlighting some limitations of current offline MARL algorithms in addressing OOD joint actions and sub-optimal convergence issues. To overcome these challenges, we propose a novel algorithm called InSPO, which utilizes sequential-update in-sample learning to avoid OOD joint actions, and introduces policy entropy to ensure comprehensive exploration of the dataset, thus avoiding the influence of local optimum behavior policies. Furthermore, we theoretically demonstrate that InSPO possesses monotonic improvement and QRE convergence properties, and then empirically validate its superior performance on various MARL benchmarks. 
For future research, integrating sequential-update in-sample learning and enhanced dataset utilization with other offline MARL algorithms presents an intriguing direction.


\section{Acknowledgments}

We gratefully acknowledge the support from the National Natural Science Foundation of China (No. 62076259, 62402252), the Fundamental and Applicational Research Funds of Guangdong Province (No. 2023A1515012946), the Fundamental Research Funds for the Central Universities Sun Yat-sen University, and the Pengcheng Laboratory Project (PCL2023A08, PCL2024Y02). This research is also supported by Meituan.


\appendix

\setcounter{secnumdepth}{2}

\section{Proofs}

\subsection{Proof of Policy Evaluation}

\begin{lemma}
    Given a policy $\boldsymbol{\pi}$, consider the modified policy evaluation operator $\mathcal{T}_{\boldsymbol{\pi}}$ under MEBR-MGs and a initial Q-function $\boldsymbol{Q}_0: \mathcal{S}\times\mathcal{A}\rightarrow\mathbb{R}$, and define $\boldsymbol{Q}_{k+1}=\mathcal{T}_{\boldsymbol{\pi}} \boldsymbol{Q}_{k}$. 
    Then the sequence $\boldsymbol{Q}_{k}$ will converge to the Q-function $\boldsymbol{Q}_{\boldsymbol{\pi}}$ of policy $\boldsymbol{\pi}$ as $k\rightarrow\infty$. 
\end{lemma}
\begin{proof}
    With a pseudo-reward $r_{\boldsymbol{\pi}}(s, \boldsymbol{a})\triangleq r(s, \boldsymbol{a}) - \gamma\mathbb{E}_{s^\prime|s,\boldsymbol{a}}[\alpha D_\text{KL}(\boldsymbol{\pi}(\cdot|s^\prime), \boldsymbol{\mu}(\cdot|s^\prime)) - \beta \mathcal{H}(\boldsymbol{\pi}(\cdot|s^\prime))$], the update rule of Q-function can be represented as: 
    \begin{align*}
        \boldsymbol{Q}(s, \boldsymbol{a})
        \gets
        r_{\boldsymbol{\pi}}(s, \boldsymbol{a})
        +\gamma\mathbb{E}_{s^\prime|s,\boldsymbol{a}, \boldsymbol{a}^\prime\sim\boldsymbol{\pi}}\Big[ \boldsymbol{Q}(s^\prime, \boldsymbol{a}^\prime) \Big].
    \end{align*}
    Then, we can apply the standard convergence results for policy evaluation~\cite{sutton2018reinforcement}. 
\end{proof}

\subsection{Proof of QRE}

\begin{proposition}\label{re prop: representation of optimal policy}
    In a MEBR-MG, a joint policy $\boldsymbol{\pi}_*$ is a QRE if it holds
    \begin{align}
        V_{\boldsymbol{\pi}_{*}}(s) \ge V_{\pi^i, \boldsymbol{\pi}_{*}^{-i}}(s), 
        \quad \forall{i}\in\mathcal{N}, \pi^i, s \in \mathcal{S}. 
    \end{align}
    Then the QRE policies for each agent $i$ are given by
    \begin{align}\label{re eq: QRE}
        \pi^{i}_*(a^i|s) &\propto 
        \mu^i(a^i|s) 
        \nonumber\\
        \cdot\exp &\Big( \frac{\mathbb{E}_{\boldsymbol{a}^{-i}\sim\boldsymbol{\pi}^{-i}_{*}}[\boldsymbol{Q}_{\boldsymbol{\pi}_*}(s, a^i, \boldsymbol{a}^{-i})] - \beta \log \mu^i(a^i|s)}{\alpha + \beta} \Big)
    \end{align}
\end{proposition}

\begin{proof}

We consider the following constrained policy optimization problem to agent $i$:
\begin{align*}
    {\max}_{\pi^{i}} &\mathbb{E}_{a^{i}\sim\pi^{i}, \boldsymbol{a}^{-i}\sim\boldsymbol{\pi}^{-i}}\Big[ 
    \boldsymbol{Q}_{\boldsymbol{\pi}}(s, \boldsymbol{a}) \Big]
    \\&- 
    \alpha \sum_{j=1}^{N}\sum_{a^j}\pi^j(a^j|s)\log\frac{\pi^{j}(a^{j}|s)}{\mu^{j}(a^{j}|s)}
    \\&- 
    \beta \sum_{j=1}^{N}\sum_{a^j}\pi^j(a^j|s)\log{\pi^{j}(a^{j}|s)},
    \\\text{s.t.}&
    \sum_{a^i}\pi^i(a^i|s) = 1, \quad \forall{s}\in\mathcal{S}. 
\end{align*}
Its associated Lagrangian function is
\begin{align*}
    \mathcal{L}(\pi^i, \lambda) = 
    &\mathbb{E}_{a^{i}\sim\pi^{i}, \boldsymbol{a}^{-i}\sim\boldsymbol{\pi}^{-i}}\Big[ 
    \boldsymbol{Q}_{\boldsymbol{\pi}}(s, \boldsymbol{a}) \Big]
    \\&- 
    \alpha \sum_{j=1}^{N}\sum_{a^j}\pi^j(a^j|s)\log\frac{\pi^{j}(a^{j}|s)}{\mu^{j}(a^{j}|s)}
    \\&- 
    \beta \sum_{j=1}^{N}\sum_{a^j}\pi^j(a^j|s)\log{\pi^{j}(a^{j}|s)}
    \\&+
    \lambda \Big( \sum_{a^i}\pi^i(a^i|s) - 1 \Big). 
\end{align*}

Therefore, we have:
\begin{align*}
    \frac{\partial \mathcal{L}(\pi^i, \lambda)}{\partial \pi^i(a^i|s)}
    &=
    \mathbb{E}_{\boldsymbol{a}^{-i}\sim\boldsymbol{\pi}^{-i}} \big[\boldsymbol{Q}_{\boldsymbol{\pi}}(s, a^i, \boldsymbol{a}^{-i}) \big]
    \\- \alpha& \log \frac{\pi^i(a^i|s)}{\mu^i(a^i|s)} - \beta \log \pi^i(a^i|s) - \alpha - \beta + \lambda.
\end{align*}
According to the Karush-Kuhn-Tucker (KKT) conditions, we know the optimal policy (i.e., QRE policy) $\pi^i_*$ satisfies
\begin{align*}
    &\pi_*^i(a^i|s) = 
    \exp(\frac{\lambda_*}{\alpha+\beta} - 1)\cdot \mu^{i}(a^i|s)\cdot\exp\Big( 
    \\&
    \frac{\mathbb{E}_{\boldsymbol{a}^{-i}\sim\boldsymbol{\pi}^{-i}_{*}}[\boldsymbol{Q}_{\boldsymbol{\pi}_*}(s, a^i, \boldsymbol{a}^{-i})] - \beta \log \mu^i(a^i|s)}{\alpha + \beta} \Big),
\end{align*}
where $\lambda_*$ is used to ensure $\sum_{a^i}\pi^i_*(a^i|s) = 1$, i.e.,
\begin{align*}
    &\exp(1 - \frac{\lambda_*}{\alpha+\beta})
    =\sum_{a^i}\Big[
    \mu^{i}(a^i|s)
    \\&\cdot\exp\Big( 
    \frac{\mathbb{E}_{\boldsymbol{a}^{-i}\sim\boldsymbol{\pi}^{-i}_{*}}[\boldsymbol{Q}_{\boldsymbol{\pi}_*}(s, a^i, \boldsymbol{a}^{-i})] - \beta \log \mu^i(a^i|s)}{\alpha + \beta} \Big)
    \Big].
\end{align*}
Thus, the proof is completed.

\end{proof}

\subsection{Proof of Policy Improvement}

\begin{proposition}\label{re prop: policy improvement}
    The sequential policy optimization procedure under MEBR-MG guarantees policy improvement, i.e., $\forall{s}\in\mathcal{S}, a\in\mathcal{A}$, 
    \begin{align*}
        &\boldsymbol{Q}_{\boldsymbol{\pi}_\text{new}}(s, \boldsymbol{a})
        \ge
        \boldsymbol{Q}_{\boldsymbol{\pi}_\text{old}}(s, \boldsymbol{a}) 
        ,
        V_{\boldsymbol{\pi}_\text{new}}(s) \ge V_{\boldsymbol{\pi}_\text{old}}(s). 
    \end{align*}
\end{proposition}

\begin{proof}

The policy improvement step of sequential policy optimization is given by
\begin{align}\label{re eq: sequential-update objective 1}
    \pi_\text{new}^{i_n}&={\arg\max}_{\pi^{i_n}} \mathbb{E}_{a^{i_n}\sim\pi^{i_n}}\Big[ 
    Q^{i_{1:n}}_{\boldsymbol{\pi}_\text{old}}(s, a^{i_n}) 
    \nonumber\\&- 
    \alpha \log\frac{\pi^{i_n}(a^{i_n}|s)}{\mu^{i_n}(a^{i_n}|s)}  - \beta \log \pi^{i_n}(a^{i_n}|s) \Big],
\end{align}
 where
\begin{align*}
    Q^{i_{1:n}}_{\boldsymbol{\pi}_\text{old}}(s, a^{i_n})
    \triangleq
    \mathbb{E}_{\boldsymbol{\pi}_\text{new}^{i_{1:n-1}},\boldsymbol{\pi}_\text{old}^{-i_{1:n}}}\Big[ \boldsymbol{Q}_{\boldsymbol{\pi_\text{old}}}(s, \boldsymbol{a}^{-i_{n}}, a^{i_n}) \Big]. 
\end{align*}

We first show that the resulting joint policy $\boldsymbol{\pi}_\text{new}$ satisfies:
\begin{align*}
    \boldsymbol{\pi}_\text{new}=&{\arg\max}_{\boldsymbol{\pi}} \mathbb{E}_{\boldsymbol{a}\sim\boldsymbol{\pi}}\Big[ 
    Q_{\boldsymbol{\pi}_\text{old}}(s,\boldsymbol{a}) 
    \nonumber\\&- 
    \alpha \log\frac{\boldsymbol{\pi}(\boldsymbol{a}|s)}{\boldsymbol{\mu}(\boldsymbol{a}|s)}  - \beta \log \boldsymbol{\pi}(\boldsymbol{a}|s) \Big]
    \\
    =&{\arg\max}_{\boldsymbol{\pi}} L_{\boldsymbol{\pi}_\text{old}}(\boldsymbol{\pi}).
\end{align*}
Otherwise, suppose that there exists a policy $\boldsymbol{\bar\pi} \neq \boldsymbol{\pi}_\text{new}$ such that $L_{\boldsymbol{\pi}_\text{old}}(\boldsymbol{\bar\pi}) > L_{\boldsymbol{\pi}_\text{old}}(\boldsymbol{\pi}_\text{new})$.
From the policy improvement step, we have 
\begin{align*}
    \mathbb{E}_{a^{i_n}\sim\pi_\text{new}^{i_n}}&\Big[ 
    \mathbb{E}_{\boldsymbol{\pi}_\text{new}^{i_{1:n-1}},\boldsymbol{\pi}_\text{old}^{-i_{1:n}}}\Big[ \boldsymbol{Q}_{\boldsymbol{\pi_\text{old}}}(s, \boldsymbol{a}^{-i_{n}}, a^{i_n}) \Big] 
    \\-&
    \alpha \log\frac{\pi_\text{new}^{i_n}(a^{i_n}|s)}{\mu^{i_n}(a^{i_n}|s)}  - \beta \log \pi_\text{new}^{i_n}(a^{i_n}|s) \Big]
    \\\ge&
    \\
    \mathbb{E}_{a^{i_n}\sim\bar\pi^{i_n}}&\Big[ 
    \mathbb{E}_{\boldsymbol{\pi}_\text{new}^{i_{1:n-1}},\boldsymbol{\pi}_\text{old}^{-i_{1:n}}}\Big[ \boldsymbol{Q}_{\boldsymbol{\pi_\text{old}}}(s, \boldsymbol{a}^{-i_{n}}, a^{i_n}) \Big] 
    \\-&
    \alpha \log\frac{\bar\pi^{i_n}(a^{i_n}|s)}{\mu^{i_n}(a^{i_n}|s)}  - \beta \log \bar\pi^{i_n}(a^{i_n}|s) \Big]
\end{align*}
Subtracting both sides of the inequality by $\mathbb{E}_{\boldsymbol{\pi}_\text{new}^{i_{1:n-1}},\boldsymbol{\pi}_\text{old}^{-i_{1:n-1}}}\Big[ \boldsymbol{Q}_{\boldsymbol{\pi_\text{old}}}(s, \boldsymbol{a}^{-i_{n}}, a^{i_n}) \Big] $ gives
\begin{align}\label{eq: use in policy improvement}
    \mathbb{E}_{a^{i_n}\sim\pi_\text{new}^{i_n}}&\Big[ 
    \mathbb{E}_{\boldsymbol{\pi}_\text{new}^{i_{1:n-1}}}\Big[ \boldsymbol{A}^{i_{n}}_{\boldsymbol{\pi_\text{old}}}(s, \boldsymbol{a}^{i_{1:n-1}}, a^{i_n}) \Big] 
    \nonumber\\-&
    \alpha \log\frac{\pi_\text{new}^{i_n}(a^{i_n}|s)}{\mu^{i_n}(a^{i_n}|s)}  - \beta \log \pi_\text{new}^{i_n}(a^{i_n}|s) \Big]
    \nonumber\\\ge&
    \nonumber\\
    \mathbb{E}_{a^{i_n}\sim\bar\pi^{i_n}}&\Big[ 
    \mathbb{E}_{\boldsymbol{\pi}_\text{new}^{i_{1:n-1}}}\Big[ \boldsymbol{A}^{i_n}_{\boldsymbol{\pi_\text{old}}}(s, \boldsymbol{a}^{i_{1:n-1}}, a^{i_n}) \Big] 
    \nonumber\\-&
    \alpha \log\frac{\bar\pi^{i_n}(a^{i_n}|s)}{\mu^{i_n}(a^{i_n}|s)}  - \beta \log \bar\pi^{i_n}(a^{i_n}|s) \Big],
\end{align}
where
\begin{align*}
    \boldsymbol{A}^{i_n}_{\boldsymbol{\pi_\text{old}}}(s, \boldsymbol{a}^{i_{1:n-1}}, a^{i_n}) &\triangleq \mathbb{E}_{\boldsymbol{\pi}_\text{old}^{-i_{1:n}}}\Big[ \boldsymbol{Q}_{\boldsymbol{\pi_\text{old}}}(s, \boldsymbol{a}^{-i_{n}}, a^{i_n}) \Big]
    \\-&
    \mathbb{E}_{\boldsymbol{\pi}_\text{old}^{-i_{1:n-1}}}\Big[ \boldsymbol{Q}_{\boldsymbol{\pi_\text{old}}}(s, \boldsymbol{a}^{-i_{n}}, a^{i_n}) \Big].
\end{align*}
Combining this inequality (\ref{eq: use in policy improvement}) with Lemma~\ref{lemma: Multi-Agent Advantage Decomposition}, we have:
\begin{align*}
    &\mathbb{E}_{\boldsymbol{a}\sim\boldsymbol{\pi}_\text{new}}\Big[ 
    \boldsymbol{A}_{\boldsymbol{\pi}_\text{old}}(s,\boldsymbol{a}) 
    - 
    \alpha \log\frac{\boldsymbol{\pi}_\text{new}(\boldsymbol{a}|s)}{\boldsymbol{\mu}(\boldsymbol{a}|s)}  - \beta \log \boldsymbol{\pi}_\text{new}(\boldsymbol{a}|s) \Big]
    \\=&
    \sum_{n=1}^{N}
    \mathbb{E}_{\boldsymbol{a}^{i_{1:n-1}}\sim\boldsymbol{\pi}^{i_{1:n-1}}_\text{new}, a^{i_{n}}\sim{\pi}^{i_{n}}_\text{new}}\Big[ 
    \boldsymbol{A}^{i_n}_{\boldsymbol{\pi}_\text{old}}(s,\boldsymbol{a}^{i_{1:n-1}}, a^{i_n}) 
    \\
    &- 
    \alpha \log\frac{\pi^{i_n}_\text{new}(a^{i_n}|s)}{{\mu}^{i_n}(a^{i_n}|s)}  - \beta \log \pi^{i_n}_\text{new}(a^{i_n}|s) \Big]
    \\\ge&
    \sum_{n=1}^{N}
    \mathbb{E}_{\boldsymbol{a}^{i_{1:n-1}}\sim\boldsymbol{\pi}^{i_{1:n-1}}_\text{new}, a^{i_{n}}\sim{\bar\pi}^{i_{n}}}\Big[ 
    \boldsymbol{A}^{i_n}_{\boldsymbol{\pi}_\text{old}}(s,\boldsymbol{a}^{i_{1:n-1}}, a^{i_n}) 
    \\
    &- 
    \alpha \log\frac{\bar\pi^{i_n}(a^{i_n}|s)}{{\mu}^{i_n}(a^{i_n}|s)}  - \beta \log \bar\pi^{i_n}(a^{i_n}|s) \Big]
    \\=&
    \mathbb{E}_{\boldsymbol{a}\sim\boldsymbol{\bar\pi}}\Big[ 
    \boldsymbol{A}_{\boldsymbol{\pi}_\text{old}}(s,\boldsymbol{a}) 
    - 
    \alpha \log\frac{\boldsymbol{\bar\pi}(\boldsymbol{a}|s)}{\boldsymbol{\mu}(\boldsymbol{a}|s)}  - \beta \log \boldsymbol{\bar\pi}(\boldsymbol{a}|s) \Big].
\end{align*}
The resulting inequality can be equivalently rewritten as
\begin{align*}
    &\mathbb{E}_{\boldsymbol{a}\sim\boldsymbol{\pi}_\text{new}}\Big[ 
    \boldsymbol{Q}_{\boldsymbol{\pi}_\text{old}}(s,\boldsymbol{a}) 
    - 
    \alpha \log\frac{\boldsymbol{\pi}_\text{new}(\boldsymbol{a}|s)}{\boldsymbol{\mu}(\boldsymbol{a}|s)}  - \beta \log \boldsymbol{\pi}_\text{new}(\boldsymbol{a}|s) \Big]
    \\
    &\ge
    \mathbb{E}_{\boldsymbol{a}\sim\boldsymbol{\bar\pi}}\Big[ 
    \boldsymbol{Q}_{\boldsymbol{\pi}_\text{old}}(s,\boldsymbol{a}) 
    - 
    \alpha \log\frac{\boldsymbol{\bar\pi}(\boldsymbol{a}|s)}{\boldsymbol{\mu}(\boldsymbol{a}|s)}  - \beta \log \boldsymbol{\bar\pi}(\boldsymbol{a}|s) \Big],
\end{align*}
which contradicts the claim $L_{\boldsymbol{\pi}_\text{old}}(\boldsymbol{\bar\pi}) > L_{\boldsymbol{\pi}_\text{old}}(\boldsymbol{\pi}_\text{new})$. 
Hence, we have $\boldsymbol{\pi}_\text{new} = \arg\max_{\boldsymbol{\pi}} L_{\boldsymbol{\pi}_\text{old}}(\boldsymbol{\pi})$, which gives
\begin{align*}
    &\mathbb{E}_{\boldsymbol{a}\sim\boldsymbol{\pi}_\text{new}}\Big[ 
    \boldsymbol{Q}_{\boldsymbol{\pi}_\text{old}}(s,\boldsymbol{a}) 
    - 
    \alpha \log\frac{\boldsymbol{\pi}_\text{new}(\boldsymbol{a}|s)}{\boldsymbol{\mu}(\boldsymbol{a}|s)}  - \beta \log \boldsymbol{\pi}_\text{new}(\boldsymbol{a}|s) \Big]
    \\
    &\ge
    \mathbb{E}_{\boldsymbol{a}\sim\boldsymbol{\pi}_\text{old}}\Big[ 
    \boldsymbol{Q}_{\boldsymbol{\pi}_\text{old}}(s,\boldsymbol{a}) 
    - 
    \alpha \log\frac{\boldsymbol{\pi}_\text{old}(\boldsymbol{a}|s)}{\boldsymbol{\mu}(\boldsymbol{a}|s)}  - \beta \log \boldsymbol{\pi}_\text{old}(\boldsymbol{a}|s) \Big]
    \\&=
    \boldsymbol{V}_{\boldsymbol{\pi}_\text{old}}(s)
\end{align*}
Therefore, we have
\begin{align*}
    &\boldsymbol{Q}_{\boldsymbol{\pi}_\text{old}}(s,\boldsymbol{a}) = 
    r(s, \boldsymbol{a}) + \gamma   \mathbb{E}_{s^\prime|s, \boldsymbol{a}}\big[ \boldsymbol{V}_{\boldsymbol{\pi}_\text{old}}(s^\prime) \big]
    \\\le&
    r(s, \boldsymbol{a}) + \gamma   \mathbb{E}_{s^\prime|s, \boldsymbol{a}}\Big[ 
    \\&
    \mathbb{E}_{\boldsymbol{a}\sim\boldsymbol{\pi}_\text{new}}\Big[ 
    \boldsymbol{Q}_{\boldsymbol{\pi}_\text{old}}(s^\prime,\boldsymbol{a}) 
    - 
    \alpha \log\frac{\boldsymbol{\pi}_\text{new}(\boldsymbol{a}|s^\prime)}{\boldsymbol{\mu}(\boldsymbol{a}|s^\prime)}  - \beta \log \boldsymbol{\pi}_\text{new}(\boldsymbol{a}|s^\prime) \Big]
    \Big]
    \\\vdots&
    \\\le&
    \boldsymbol{Q}_{\boldsymbol{\pi}_\text{new}}(s,\boldsymbol{a}), \quad \forall{s, \boldsymbol{a}}
\end{align*}
And then, 
\begin{align*}
    &\boldsymbol{V}_{\boldsymbol{\pi}_\text{old}}(s) 
    \\&= 
    \mathbb{E}_{\boldsymbol{a}\sim\boldsymbol{\pi}_\text{old}}\Big[ 
    \boldsymbol{Q}_{\boldsymbol{\pi}_\text{old}}(s,\boldsymbol{a}) 
    - 
    \alpha \log\frac{\boldsymbol{\pi}_\text{old}(\boldsymbol{a}|s)}{\boldsymbol{\mu}(\boldsymbol{a}|s)}  - \beta \log \boldsymbol{\pi}_\text{old}(\boldsymbol{a}|s) \Big]
    \\&\le
    \mathbb{E}_{\boldsymbol{a}\sim\boldsymbol{\pi}_\text{new}}\Big[ 
    \boldsymbol{Q}_{\boldsymbol{\pi}_\text{old}}(s,\boldsymbol{a}) 
    - 
    \alpha \log\frac{\boldsymbol{\pi}_\text{new}(\boldsymbol{a}|s)}{\boldsymbol{\mu}(\boldsymbol{a}|s)}  - \beta \log \boldsymbol{\pi}_\text{new}(\boldsymbol{a}|s) \Big]
    \\&\le
    \mathbb{E}_{\boldsymbol{a}\sim\boldsymbol{\pi}_\text{new}}\Big[ 
    \boldsymbol{Q}_{\boldsymbol{\pi}_\text{new}}(s,\boldsymbol{a}) 
    - 
    \alpha \log\frac{\boldsymbol{\pi}_\text{new}(\boldsymbol{a}|s)}{\boldsymbol{\mu}(\boldsymbol{a}|s)}  - \beta \log \boldsymbol{\pi}_\text{new}(\boldsymbol{a}|s) \Big]
    \\&=
    \boldsymbol{V}_{\boldsymbol{\pi}_\text{new}}(s),\quad \forall{s}.
\end{align*}
Thus, the proof is completed.

\end{proof}

\begin{lemma}[Multi-Agent Advantage Decomposition from \citeauthor{HAPPO/KubaCWWSW022}]\label{lemma: Multi-Agent Advantage Decomposition}
    For any state s and joint action, the joint advantage function can be decomposed as:
    \begin{align*}
        \boldsymbol{A}_{\boldsymbol{\pi}}(s, \boldsymbol{a})
        =
        \sum_{n=1}^{N}\boldsymbol{A}^{i_n}_{\boldsymbol{\pi}}(s, \boldsymbol{a}^{i_{1:n-1}}, a^{i_n})
    \end{align*}.
\end{lemma}
\begin{proof}
    By the definition of multi-agent advantage function, we have:
    \begin{align*}
        \boldsymbol{A}_{\boldsymbol{\pi}}(s, \boldsymbol{a})
        =&
        \boldsymbol{Q}_{\boldsymbol{\pi}}(s, \boldsymbol{a}) - \boldsymbol{V}_{\boldsymbol{\pi}}(s)
        \\=&
        \sum_{n=1}^{N}\Big[
        \mathbb{E}_{\boldsymbol{\pi}^{-i_{1:n}}}\Big[ \boldsymbol{Q}_{\boldsymbol{\pi}}(s, \boldsymbol{a}^{-i_{n}}, a^{i_n}) \Big]
        \\&-
        \mathbb{E}_{\boldsymbol{\pi}^{-i_{1:n-1}}}\Big[ \boldsymbol{Q}_{\boldsymbol{\pi}}(s, \boldsymbol{a}^{-i_{n}}, a^{i_n}) \Big]
        \Big]
        \\=&
        \sum_{n=1}^{N}\boldsymbol{A}^{i_n}_{\boldsymbol{\pi}}(s, \boldsymbol{a}^{i_{1:n-1}}, a^{i_n})
    \end{align*}.
\end{proof}

\subsection{Proof of QRE convergence}

\begin{theorem}\label{re thrm: converge to QRE}
    In the tabular setting, the joint policy $\boldsymbol{\pi}$ updated by InSPO converges to QRE. 
\end{theorem}
\begin{proof}
First, we have that $\boldsymbol{Q}_{\boldsymbol{\pi}_{k+1}}(s, \boldsymbol{a}) \ge \boldsymbol{Q}_{\boldsymbol{\pi}_{k}}(s, \boldsymbol{a})$ by Proposition~\ref{re prop: policy improvement} and that the Q-function is upper-bounded since reward and regularizations are bounded. Hence, the sequence of policies converges to some limit point $\boldsymbol{\bar \pi}$. 

Then, considering this limit point joint policy $\boldsymbol{\bar \pi}$, it must be the case that $\forall{i, \pi^{i}}$,
\begin{align*}
    \mathbb{E}_{a^{i}\sim\bar \pi^{i}}&\Big[ 
    \mathbb{E}_{\boldsymbol{\bar \pi}^{-i}}\Big[ \boldsymbol{Q}_{\boldsymbol{\bar \pi}}(s, \boldsymbol{a}^{-i}, a^{i}) \Big] 
    \\-&
    \alpha \log\frac{\bar \pi^{i}(a^{i}|s)}{\mu^{i}(a^{i}|s)}  - \beta \log \bar\pi^{i}(a^{i}|s) \Big]
    \\\ge&
    \\
    \mathbb{E}_{a^{i}\sim \pi^{i}}&\Big[ 
    \mathbb{E}_{\boldsymbol{\bar \pi}^{-i}}\Big[ \boldsymbol{Q}_{\boldsymbol{\bar \pi}}(s, \boldsymbol{a}^{-i}, a^{i}) \Big] 
    \\-&
    \alpha \log\frac{ \pi^{i}(a^{i}|s)}{\mu^{i}(a^{i}|s)}  - \beta \log \pi^{i}(a^{i}|s) \Big],
\end{align*}
which is equivalent with $V_{\boldsymbol{\bar\pi}}(s) \ge V_{\pi^i, \boldsymbol{\bar\pi}^{-i}}(s)$. 
Thus, $\boldsymbol{\bar \pi}$ is a quantal response equilibrium, which finishes the proof.

\end{proof}


\section{Details of Practical Algorithm}


In the practical implementation of InSPO, we train the policy network $\theta^{i_n}$ by loss function
\begin{align}\label{re eq: final in-sample}
    &J(\theta^{i_n}) \triangleq
    \mathbb{E}_{(s,a^{i_n})\sim\mathcal{D}_{\rho^{i_n}}}\Big[
    \nonumber\\&
    -\exp\Big(
    \frac{A_{\bar\phi^{i_n}}(s, a^{i_n}) - \beta \log \mu^{i_n}(a^{i_n}|s)}{\alpha+\beta}
    \Big)
    \log \pi_{\theta^{i_n}}(a^{i_n}|s)
    \Big],
\end{align}
where
\begin{align*}
    A_{\bar\phi^{i_n}}(s, a^{i_n})
    \triangleq
    Q_{\bar\phi^{i_n}}(s, a^{i_n}) - 
    \mathbb{E}_{\pi_{\theta^{i_n}}} [ Q_{\bar\phi^{i_n}}(s, a^{i_n}) ],
\end{align*}
and $\bar\phi^{i_n}$ is the soft-target network of $\phi^{i_n}$. 

The behavior policy $\mu^{i_n}$ in $J(\theta^{i_n})$ is pre-trained by the Behavior Cloning
\begin{align}\label{eq: BC}
    \min_{\mu^{i_n}}\mathbb{E}_{(s, a^{i_n})\sim\mathcal{D}}\Big[
    -\log \mu^{i_n}(a^{i_n}|s)
    \Big]. 
\end{align}

Instead of using the $\mu^{i_n}$ trained by Eq.(\ref{eq: BC}) to calculate $\boldsymbol{\mu}^{-i_n}$ in the computation of importance resampling ratio $\rho^{i_n}$, we pre-train an MLP-based autoregressive behavior policy, same with the one in \citeauthor{alberdice/MatsunagaLYLAK23}, for numerically stability. 
The optimization objective of the autoregressive behavior policy is given by
\begin{align}\label{eq: auto behavior}
    \min_{\boldsymbol{\mu}^{-i}}
    \mathbb{E}_{(s, \boldsymbol{a})\sim\mathcal{D}}\Big[ -\sum_{j=1, j\neq i}^{N}\log \boldsymbol{\mu}^{-i}(a^j|s, a^i, \boldsymbol{a}^{1:i-1}) \Big]. 
\end{align}
Then, in the computation of $\rho^{i_n}$, we use the geometric mean to prevent the collapse of $\rho^{i_n}$ due to the growth of the number of agents: 
\begin{align}\label{eq: IR ratio}
    \rho^{i_n} = \Big(\frac{\boldsymbol{\pi}^{-{i_n}}(\boldsymbol{a}^{-i_n}|s)}{\boldsymbol{\mu}^{-{i_n}}(\boldsymbol{a}^{-i_n}|s)}\Big)^{\frac{1}{N-1}} . 
\end{align}

In order to estimate the future return based on state-action pair, we train the local Q-function network by minimizing the temporal difference (TD) error with a CQL regularization term to further penalize OOD action values:
\begin{align}\label{eq: local q-function}
    J(\phi^{i_n}) &\triangleq
    \mathbb{E}_{(s,\boldsymbol{a}, s^\prime, r)\sim\mathcal{D}_{\rho^{i_n}}}
    \Big[
    \Big(
    Q_{\phi^{i_n}}(s, a^{i_n}) - 
    y
    \Big)^2\Big]
    \nonumber\\&+
    \alpha_\text{CQL} 
    \mathbb{E}_{s\sim\mathcal{D}_{\rho^{i_n}}}
    \Big[
    \log\sum_{a^{i_n}}\exp( Q_{\phi^{i_n}}(s, a^{i_n}) )
    \nonumber\\&- 
    \mathbb{E}_{a^{i_n}\sim\mu^{i_n}} \big[ Q_{\phi^{i_n}}(s, a^{i_n}) \big]
    \Big],
\end{align}
where
\begin{align*}
    y =& y(s,\boldsymbol{a}, s^\prime, r) \triangleq r + \gamma \mathbb{E}_{a^{i_{n}\prime}\sim\pi_\text{old}^{i_n}} \big[Q_{{\bar{\phi}}^{i_n}}(s^{\prime}, a^{i_n \prime}) \big]
    \\&-
    \alpha D_\text{KL}(\pi_{\theta^{i_n}}(\cdot | s^\prime), \mu^{i_n}(\cdot | s^\prime)) + \beta \mathcal{H}(\pi_{\theta^{i_n}}(\cdot | s^\prime)). 
\end{align*}
Here we use $\alpha_\text{CQL}=0.1$ as the default value. 

Lastly, we give the loss function of auto-tuned $\alpha$, which is inspired by the auto-tuned temperature extension of SAC~\cite{autoSAC}:
\begin{align}
    J(\alpha) \triangleq
    \mathbb{E}_{s\sim\mathcal{D}}\Big[\sum_{i\in\mathcal{N}}\Big( \alpha D_\text{KL}( \pi^{i}(\cdot|s), \mu^{i}(\cdot|s)) - \alpha \bar D_\text{KL}\Big)\Big],
\end{align}
where $\bar D_\text{KL}$ is the target value with the default value $0.18$.


\section{Experimental Details}

\begin{table*}[ht!]
    \small
    \centering
    \begin{tabular}{c c|c c c}
         Map
         &Dataset
         &random
         &fixed
         &semi-greedy
         \\
         \hline
         \multirow{4}*{5m\_vs\_6m}
         &medium
         &$0.28\pm 0.06$
         &$0.29\pm 0.06$
         &$0.28\pm 0.03$
         \\
         ~
         &medium-replay
         &$0.24\pm 0.07$
         &$0.25\pm 0.10$
         &$0.27\pm 0.04$
         \\
         ~
         &expert
         &$0.79\pm 0.12$
         &$0.84\pm 0.05$
         &$0.81\pm 0.05$
         \\
         ~
         &mixed
         &$0.78\pm 0.06$
         &$0.78\pm 0.05$
         &$0.78\pm 0.08$
         \\
         \hline
         \multirow{4}*{6h\_vs\_8z}
         &medium
         &$0.43\pm0.06$
         &$0.39\pm 0.07$
         &$0.42\pm 0.11$
         \\
         ~
         &medium-replay
         &$0.23\pm 0.02$
         &$0.23\pm 0.04$
         &$0.17\pm 0.08$
         \\
         ~
         &expert
         &$0.74\pm 0.11$
         &$0.72\pm 0.07$
         &$0.71\pm 0.09$
         \\
         ~
         &mixed
         &$0.60\pm 0.12$
         &$0.62\pm 0.10$
         &$0.65\pm 0.11$
    \end{tabular}
    \caption{Impact of update order. }
    \label{tab: impact of update order}
\end{table*}

\subsection{Baselines}

\textbf{BC}, \textbf{OMAR} and \textbf{CFCQL}: We use the open-source implementation~\footnote{https://github.com/thu-rllab/CFCQL\label{cfcql}} provided by \citeauthor{cfcql/ShaoQCZJ23}. 
\textbf{AlberDICE}: We use the open-source implementation~\footnote{https://github.com/dematsunaga/alberdice\label{alberdice}} provided by \citeauthor{alberdice/MatsunagaLYLAK23}. 
\textbf{OMIGA}: We use the open-source implementation~\footnote{https://github.com/ZhengYinan-AIR/OMIGA} provided by \citeauthor{omiga/WangXZZ23}.

\subsection{Dataset Details}

\paragraph{Bridge. } 
We use the datasets~\textsuperscript{\ref{alberdice}} provided \citeauthor{alberdice/MatsunagaLYLAK23}.
The optimal dataset (500 trajectories) is collected by a hand-crafted (multi-modal) optimal policy, and the mixed dataset (500 trajectories) is the equal mixture of the optimal dataset and 500 trajectories collected by a uniform random policy.

\paragraph{StarCraft II. } 
We use the datasets~\textsuperscript{\ref{cfcql}} provided \citeauthor{cfcql/ShaoQCZJ23}, which are collected by QMIX~\cite{QMIX/RashidSWFFW18}. 
The medium dataset (5000 trajectories) is collected by a partially trained model of QMIX, and the expert dataset (5000 trajectories) is collected by a fully trained model. 
The mixed dataset (5000 trajectories) is the equal mixture of medium and expert datasets. 
The medium-replay dataset (5000 trajectories) is the replay buffer during training until the policy reaches the medium performance.

\subsection{Resources}

We run all the experiments on 4*NVIDIA GeForce RTX 3090 GPUs and 4*NVIDIA A30. Each setting is repeated for 5 seeds. 
For one seed, it takes about 1.5 hours for StarCraft II, 1 hour for Bridge, and 15 minutes for matrix games.

\subsection{Reproducibility}

The local Q-function network is represented by 3-layer ReLU activated MLPs with 256 units for each hidden layer. 
The policy network is implemented in two ways: MLP and RNN.
For the Matrix Game and Bridge, the policy network is represented by 3-layer ReLU activated MLPs with 256 units for each hidden layer. 
For StarCraft II, the policy network consists of two linear layers of 64 units and one GRUCell layer, referring to the CFCQL implementation~\textsuperscript{\ref{cfcql}}. 
All the networks are optimized by Adam optimizer. 

For all datasets and algorithms, we run all the experiments with 1e7 training steps on 5 seeds: 0, 1, 2, 3, 4. For evaluation we rolled out policies for 32 episodes and computed the mean episode return (or winning rate). 
For OMAR, we tune a best CQL weight from $\{0, 0.1, 0.5, 1, 2, 3, 4, 5, 10\}$. 
For CFCQL, we tune a best CQL weight from $\{0, 0.1, 0.5, 1, 5, 10, 50\}$, softmax temperature from $\{0, 0.1, 0.5, 1, 100\}$. 
For OMIGA, we tune a best regularization temperature from $\{0.1, 0.5, 1, 3, 5, 7, 10 \}$. 
For InSPO, we tune a best $\alpha$ from $\{0.1, 0.5, 1, 3, 5 \}$, $\bar D_\text{KL}$ from $\{0.08, 0.18, 0.3\}$ and an exponentially decaying $\beta$ from $\{0, 5, 10\}$.

\subsection{Impact of Update Order}

Update order might impact performance~\cite{MAS_via_communication, orderMatter}. \citeauthor{MAS_via_communication} used a world model to determine the order by comparing the value of intentions. \citeauthor{orderMatter} introduced a semi-greedy agent selection rule, which prioritizes updating agents with a higher expected advantage. However, they are heuristic and may not guarantee optimal results. To investigate this, we add an ablation study, comparing ‘random’, ‘fixed’, and semi-greedy update rule in terms of their impact on InSPO's performance. Table~\ref{tab: impact of update order} shows that the effect of update order on performance is not significant, but determining the optimal update sequence is still an interesting direction for future work. 

\subsection{Training Efficiency}

The sequential updates can increase training times, which is a common issue with this method. 
Here we add a comparison of training times between InSPO and CFCQL in Table~\ref{tab: training efficiency}, which shows that sequential updates do increase the training time. 
However, some work~\cite{orderMatter} aim to address this issue by grouping agents into blocks for simultaneous updates within blocks, with sequential updates between blocks. 
\begin{table}[ht]
    \small
    \centering
    \begin{tabular}{c|c c}
         map
         &CFCQL
         &InSPO
         \\
         \hline
         2s3z
         &$1$h
         &$1.75$h
         \\
         3s\_vs\_5z
         &$1$h
         &$1.2$h
         \\
         5m\_vs\_6m
         &$0.5$h
         &$1.25$h
         \\
         6h\_vs\_8z
         &$0.75$h
         &$2.25$h
    \end{tabular}
    \caption{Comparison of training times between InSPO and CFCQL. }
    \label{tab: training efficiency}
\end{table}


\section{Value Decomposition in XOR}

In XOR game, the minimization of TD error $\mathbb{E}_{\mathcal{D}}[(\boldsymbol{Q}(a^1, a^2) - r(a^1, a^2))^2]$ motivates the global Q-network to satisfy
\begin{align}\label{eq: value decomposition in XOR}
    \boldsymbol{Q}(A, B) = \boldsymbol{Q}(B, A) > \boldsymbol{Q}(A, A).
\end{align}
According to IGM principle, if $Q^i(A) > Q^i(B)$, then $\boldsymbol{Q}(A, A) > \boldsymbol{Q}(A, B)$, and $\boldsymbol{Q}(A, A) > \boldsymbol{Q}(B, B)$, contradicting Eq.(\ref{eq: value decomposition in XOR}). 
If $Q^i(A) = Q^i(B)$, then $\boldsymbol{Q}(A, A) = \boldsymbol{Q}(A, B) = \boldsymbol{Q}(B, B)$, contradicting Eq.(\ref{eq: value decomposition in XOR}) again. 
Therefore, we have $Q^i(A) < Q^i(B)$, which leads to the OOD joint action $(B, B)$.


\section{Similar Techniques to the Concept of ``sequential"}


The idea of ``sequential" has been explored in various directions within MARL. 
Specifically, \citeauthor{MAS_via_communication} introduced SeqComm, a communication framework where agents condition their actions based on the \textit{ordered actions} of others, mitigating circular dependencies that arise in \textit{simultaneous} communication. MACPF~\cite{{MoreCenStillDec}} decomposes joint policies into individual ones, incorporating a correction term to model dependencies on preceding agents' actions in \textit{sequential execution}. BPPO~\cite{BackpropagationThroughAgents} employs an auto-regressive joint policy with a \textit{fixed execution order}. During training, agents act \textit{sequentially} based on the prior agents’ actions, and update their policies based on the feedback from subsequent agents. This bidirectional mechanism enables each agent adapts to the changing behavior of the team efficiently. 

While these works use the concept of ``sequential" to improve coordination, applying them directly in offline MARL poses challenges. Both policy and value functions in MACPF and BPPO explicitly condition on prior agents’ actions, which can be challenging in offline settings where required data may be missing, potentially hindering accurate value function updates.

\end{document}